\documentclass[11pt]{article}

\usepackage[margin=1in]{geometry}

\usepackage{amsmath}
\usepackage{amssymb}
\usepackage{amsthm}
\usepackage[round]{natbib}

\usepackage{url}

\usepackage{color}
\usepackage{graphicx}
\usepackage{float}
\usepackage{multirow}

\usepackage{datetime}
\usepackage{mathtools}
\mathtoolsset{showonlyrefs}
\usepackage{enumerate}

\usepackage{hyperref}
\def\name{\normalsize\bfseries}
\def\email{\quad\normalfont\small}
\def\addr{\normalfont\itshape\small}
\newenvironment{keywords}%
  {\par\vskip 0.5em\noindent\textbf{Keywords:\ }}%
  {\par}
\long\def\acks#1{\vskip 0.3in\noindent\textbf{Acknowledgments and Disclosure of
  Funding}\vskip 0.15in\noindent #1}

\newtheorem{theorem}{Theorem}
\newtheorem{lemma}[theorem]{Lemma}
\newtheorem{proposition}[theorem]{Proposition}

\newtheorem{definition}[theorem]{Definition}
\newtheorem{remark}[theorem]{Remark}

\newcommand{\fbi}{f^{\backslash i}}
\newtheorem{assumption}[theorem]{Assumption}

\newcommand{\repofootnote}{Code to reproduce all experiments is available at \url{https://github.com/adam-oberman/fast-rates-ssl}.}

\begin{document}

\title{Fast Rates for Semi-Supervised Learning via\\ Data-Augmentation Graph Regularization}

\author{\name Adam M. Oberman \\ \email adam.oberman@mcgill.ca \\
       \addr Department of Mathematics and Statistics, McGill University\\
       \addr  Mila, Quebec AI Institute\\
       \addr LawZero
       }

\date{July 2026}

\maketitle

\begin{abstract}%
Self-supervised learning matches supervised accuracy from a fraction of the labels, but the
labeled-sample efficiency behind this has lacked a theoretical explanation. We provide one. Data
augmentation induces a similarity graph on the unlabeled data, so downstream learning on that
graph is graph-Laplacian-regularized learning. We prove a fast \emph{transductive} rate,
$O(1/n_L)$ in the number of labels, in place of the supervised $O(1/\sqrt{n_L})$, by carrying the
leave-one-out stability apparatus of Johnson and Zhang (JMLR 2007) over to the augmentation graph,
and without the unrealistic assumptions of limit-based analyses (exact kernel, generalizing
features). The bound makes augmentation quality explicit: the expected error is at most
$C/n_L + R_{\mathrm{DA}}(y)$, where the data-augmentation alignment error $R_{\mathrm{DA}}(y)$ is
proportional to the graph-cut mass of augmentations that cross a label boundary, so good augmentations let few
labels suffice. The analysis uses a streamlined loss that drops the projector, negative-sample,
and orthogonality overhead of standard objectives yet still recovers the top-$K$ ideal features in
the infinite-data limit, the augmentation-kernel eigenspace studied by Zhai et al. The bound gives a mechanistic account of the accuracy-versus-label-count curve through
augmentation quality, verified in a controlled model where the constants are known.
\end{abstract}


\section{Introduction}\label{sec:intro}

Self-supervised learning matches supervised accuracy from a small fraction of the labels.
SimCLR~\citep{chen2020simclr} matches a supervised ResNet-50 from a linear probe and is strong
at $1\%$ and $10\%$ of ImageNet labels; SimCLRv2~\citep{chen2020simclrv2} reaches $80.9\%$ top-1
with only $10\%$ of the labels, an order-of-magnitude label saving; and CPC
v2~\citep{henaff2020cpcv2} named the ``data-efficient'' framing with two to five times fewer
labels. This accuracy-versus-label-count curve is one of the most reported empirical facts about
representation learning, yet it remains theoretically under-explained.

The feature-learning side of the story is by now well understood. Different self-supervised
losses (spectral contrastive, contrastive and non-contrastive, kernel-PCA) all recover the
leading eigenfunctions of the augmentation similarity kernel $k^{\mathrm{DAF}}$. But that theory
lives in the infinite-data limit: it characterizes the features a method would learn from
unlimited unlabeled data, and says nothing about the finite \emph{labeled}-sample behavior
downstream, which is what the curve above measures. The assumptions used there, that the
expectation kernel is learned exactly and that the factoring features generalize to unseen
points, are also unrealistic in the few-augmentation regime.

This paper supplies the missing labeled-sample guarantee. The starting point is that data
augmentation equips the unlabeled sample with a similarity graph, on which downstream learning
is Laplacian-regularized graph learning. For that setting a transductive fast rate is already
available: \citet{ZhangLaplacian} bound the transductive error of
Laplacian-regularized multi-class graph learning by leave-one-out algorithmic stability and obtain
an $O(1/n)$ rate in the labeled count under balanced components, governed by a learning-theoretic
graph cut. Our rate results (Theorems~\ref{thm:rate} and~\ref{thm:daerror}) are that apparatus,
specialized and reinterpreted for the augmentation graph; the contribution is not a new rate
mechanism but its placement. Read on the augmentation graph, the cut becomes an
\emph{augmentation-quality} term $R_{\mathrm{DA}}(y)$, the number of graph components becomes a
property of the augmentation distribution, and the bound turns into a statement about how few labels
a good augmentation buys, in place of the supervised $O(1/\sqrt{n_L})$. 
A related applied result, the NeurIPS paper of \citet{ghosh2024harnessing} (showing $2\times$ less pretraining data), motivates the pretraining axis.  Here we give a rate-based, mechanistic
account of the downstream labeled-sample efficiency, its mechanism checked in a controlled model
with known constants (Section~\ref{sec:expA}) and the phenomenon itself shown descriptively
on CIFAR-10 (Section~\ref{sec:cifar}).

\paragraph{Contributions.}
\begin{itemize}
  \item \textbf{A fast transductive rate on the augmentation graph (Theorem~\ref{thm:rate}).}
        An $O(1/n_L)$ oracle inequality for labels propagating from the labeled set to the
        unlabeled nodes, obtained by leave-one-out stability on the augmentation graph. Relative to the feature-learning limit theory, the argument also avoids the
        unrealistic assumptions noted above: it never requires the kernel to be learned exactly, and it bounds
        the error on the actual unlabeled sample rather than assuming features generalize.
  \item \textbf{Augmentation quality enters the bound (Theorem~\ref{thm:daerror}).} The expected
        error is bounded by $C/n_L + R_{\mathrm{DA}}(y)$, where the data-augmentation alignment
        error $R_{\mathrm{DA}}(y)$ is proportional to the graph-cut mass of augmentations that cross a label
        boundary. This term \emph{explains} accuracy: better augmentations shrink it, reducing the
        labels required.
  \item \textbf{A streamlined loss (Theorem~\ref{thm:features}).} Dropping the
        projector-dimension, negative-sample, and orthogonality / non-collapse overhead (one
        label per class already prevents collapse) costs nothing in the limit: the streamlined
        loss is algorithmically stable and still recovers the top-$K$ ideal features.
\end{itemize}

\paragraph{Consequences.}
Each ingredient of the bound is actionable.
(i) \emph{Augmentation selection.} $R_{\mathrm{DA}}(y)$ is an explicit graph cut, estimable
from a small labeled sample and the augmentation graph before the label budget is spent
(Section~\ref{sec:daerror}); candidate augmentation pipelines can be ranked by estimated cut
mass, making \eqref{eq:daerror} a selection procedure.
(ii) \emph{A simpler training objective.} By Theorem~\ref{thm:features}, the projector-dimension,
negative-sample, and orthogonality overhead of standard self-supervised losses can be dropped:
the streamlined loss is stable and recovers the same top-$K$ features, so the overhead is not
needed for the downstream guarantee.
(iii) \emph{Label-budget guidance.} The additive form of \eqref{eq:daerror} separates the
label budget from the augmentation quality: once $C/n_L$ falls below the estimated $R_{\mathrm{DA}}(y)$,
further labels buy little in the surrogate and the augmentation is the binding constraint,
while the multiplicative form predicts the $0/1$ error keeps descending
(Section~\ref{sec:expA}).
(iv) \emph{When a graph probe pays.} Theorem~\ref{thm:features} predicts that a backbone
pretrained on the matching augmentations has already performed the graph smoothing, so a
linear probe suffices there and graph regularization \eqref{eq:alg} should pay when
features were not pretrained on the augmentation defining the graph, the pattern seen in
Section~\ref{sec:cifar}. Section~\ref{sec:discussion} develops these points.

\section{Setup and notation}\label{sec:setup}

We study vector classification on a data domain $\mathcal{X}\subset\mathbb{R}^{d_0}$ with
label set $\mathcal{Y}=\{1,\dots,K\}$. A score function $g:\mathcal{X}\to\mathbb{R}^K$ is
followed by the argmax one-hot classifier; we write $e_y$ for the one-hot vector of label
$y$. The training loss is a standard classification loss $\ell(g,z)$, and for $z=(x,y)$ we
abbreviate $\ell(g,z)=\ell(g(x),y)$.

\begin{assumption}[$\sigma$-admissible loss]\label{ass:loss}
The loss $\ell(\cdot,y)$ is convex in its first argument and $\sigma_\ell$-Lipschitz on a
bounded domain $D_\ell$:
\[
|\ell(f,y)-\ell(g,y)|\le \sigma_\ell\,\|f-g\|,
\qquad \forall f,g\in D_\ell,\ \forall y\in\mathcal{Y}.
\]
\end{assumption}

\noindent This holds for the standard score-based losses (cosine-similarity, cross-entropy
on a bounded domain, regularized least squares); see Appendix~\ref{app:proofs}.

\paragraph{Transductive data.}
Let $S^U=\{x_1,\dots,x_m\}$ be the unlabeled sample of $m$ points. A subset of
$n_L\le m$ of them is labeled, giving $S^L=\{z_1,\dots,z_{n_L}\}$ with
$z_j=(x_j,y_j)$; throughout, $S^L\subset S^U$ (revealing a label does not add a point).
The labeled training loss is
\[
L(g,S^L)=\frac{1}{n_L}\sum_{z_j\in S^L}\ell(g,z_j).
\]
The setting is \emph{transductive}: error is measured on the unlabeled nodes of the fixed
sample $S^U$, not on a fresh draw.

\subsection{The data-augmentation graph}
Data augmentations are random maps $M:\mathcal{X}\to\mathcal{X}$. Writing
$p(z\mid x_0)=\mathbb{P}\!\left(z=M(x_0)\right)$ for the probability of reaching $z$ from a
seed $x_0$, the (forward) data-augmentation covariance kernel is
\begin{equation}\label{eq:kdaf}
k^{\mathrm{DAF}}(x,z)=\mathbb{E}_{x_0\sim\rho_X}\!\left[\,p(x\mid x_0)\,p(z\mid x_0)\,\right],
\end{equation}
together with its symmetric degree-normalized version. Writing
$d(x)=\mathbb{E}_{z\sim\rho_X}\!\left[k^{\mathrm{DAF}}(x,z)\right]$ for the augmentation mass
(degree) at $x$, the normalized kernel is
\begin{equation}\label{eq:knorm}
k(x,z)=\frac{k^{\mathrm{DAF}}(x,z)}{\sqrt{d(x)}\,\sqrt{d(z)}} .
\end{equation}
On the sample $S^U$ this gives the weighted graph $G=(S^U,W)$ with raw weights
$w_{ij}=k^{\mathrm{DAF}}(x_i,x_j)\ge 0$, degrees $\mathbf{S}_j=\deg_j(G)=\sum_{j'}w_{j,j'}$, and
normalized adjacency $W=\mathbf{S}^{-1/2}\,[w_{ij}]\,\mathbf{S}^{-1/2}$, so that
$W_{ij}=k(x_i,x_j)$. Convolution against $W$ realizes the integral operator $T_K$ on $\rho_X$:
applied to features on $S^U$, $W f(S^U)=T_K f$, recovering the ideal RKHS picture in the
large-sample limit. The degrees $\mathbf{S}_j$ are the scaling factors of the
$\mathbf{S}$-normalized Laplacian used below. The symmetric normalization
$W=\mathbf{S}^{-1/2}[w_{ij}]\mathbf{S}^{-1/2}$ is the geometric mean of the two random-walk
normalizations $\mathbf{S}^{-1}[w_{ij}]$ and $[w_{ij}]\mathbf{S}^{-1}$, which encode the
\emph{forward} chain (augment a seed and ask where the view lands) and the \emph{backward} chain
(given a view, ask which seed produced it); the two are adjoint, share the spectrum of $W$, and the
data-augmentation error $R_{\mathrm{DA}}(y)$ of Section~\ref{sec:daerror} is measured in this
symmetric normalization.

\subsection{The regularizer and the combined objective}
Let $\mathbf{K}\in\mathbb{R}^{m\times m}$ be a symmetric positive-definite kernel matrix
\emph{derived from the augmentation graph} $G$ (and from a fixed backbone $f_0$), so that
$\mathbf{K}$ depends on neither the labels nor the trained score $g$. The quadratic
regularizer penalizes each class score separately,
\begin{equation}\label{eq:reg}
Q(g,S^U)=\sum_{k=1}^K g_{\cdot,k}^\top\,\mathbf{K}^{-1}\,g_{\cdot,k},
\qquad g_{\cdot,k}=(g_{1,k},\dots,g_{m,k})^\top .
\end{equation}
The graph realization (used in Section~\ref{sec:daerror}) takes
$\mathbf{K}^{-1}=\alpha\,\mathbf{S}^{-1}+\mathcal{L}_{\mathbf{S}}(G)$, where
$\mathcal{L}_{\mathbf{S}}(G)=\mathbf{S}^{-1/2}\mathcal{L}(G)\,\mathbf{S}^{-1/2}$ is the
$\mathbf{S}$-normalized graph Laplacian and $\alpha>0$ keeps $\mathbf{K}$ strictly positive
definite; then
\begin{equation}\label{eq:reg-graph}
Q(g,S^U)=\sum_{k=1}^K\Big(\underbrace{\alpha\sum_j \mathbf{S}_j^{-1}g_{j,k}^2}_{\text{ridge}}
+\underbrace{\tfrac12\sum_{j,j'}w_{j,j'}\big(\mathbf{S}_j^{-1/2}g_{j,k}-\mathbf{S}_{j'}^{-1/2}g_{j',k}\big)^2}_{\text{Laplacian energy}}\Big).
\end{equation}
The second term is the invariance-to-augmentation energy. It is a weighted sum of squared
differences of the degree-normalized score across edges of the augmentation graph: an edge
weight $w_{j,j'}$ is large when $x_j$ and $x_{j'}$ are likely to arise as
augmentations of a common seed, so a small Laplacian energy means the score varies little
between points the augmentations identify, i.e.\ the score is approximately invariant to data
augmentation. This is the standard manifold / label-propagation prior: minimizing
\eqref{eq:reg-graph} pushes each class score toward functions that are smooth along the
augmentation graph, while the ridge term contributes the $\alpha\mathbf{S}^{-1}$ that keeps
$\mathbf{K}$ strictly positive definite and the regularizer proper. Evaluating this energy at
the label indicator yields the graph cut $R_{\mathrm{DA}}(y)$ of
Section~\ref{sec:daerror}, the channel through which augmentation quality enters the rate.

\begin{definition}[SSL algorithm]\label{def:alg}
Given $(S^L,S^U)$, the algorithm returns the regularized empirical minimizer
\begin{equation}\label{eq:alg}
f=A(S^L,S^U)=\arg\min_{g}\;\Big\{\,L(g,S^L)+\lambda\,Q(g,S^U)\,\Big\},
\qquad \lambda>0 .
\end{equation}
\end{definition}

\begin{remark}[No collapse term needed]
Because $\mathbf{K}$ is built from the unlabeled graph and is independent of $g$ and of the
labels, the objective \eqref{eq:alg} carries no projector-dimension, negative-sample, or
orthogonality/non-collapse penalty. With at least one label per class the supervised term
already prevents feature collapse; this is the streamlined loss analyzed in
Section~\ref{sec:features}.
\end{remark}

\section{A fast transductive rate via leave-one-out stability}\label{sec:rate}

The first result is an oracle inequality with a fast $O(1/n_L)$ dependence on the number of
labels. Error is measured by the semi-supervised leave-one-out loss: move one labeled point
to the unlabeled set, retrain, and evaluate on the held-out point, averaged over the
training set.

\begin{definition}[Leave-one-out loss]\label{def:loo}
For $f^{\backslash j}=A(S^L\setminus z_j,\,S^U)$,
\[
L_{\mathrm{loo}}(A,S^L,S^U)=\frac{1}{n_L}\sum_{z_j\in S^L}\ell\!\left(f^{\backslash j},z_j\right).
\]
\end{definition}

The argument rests on a one-point stability bound for the algorithm \eqref{eq:alg}: removing
a labeled point perturbs the fitted score at that point by only $O(1/(\lambda n_L))$.

\begin{lemma}[Uniform stability]\label{lem:stability}
Under Assumption~\ref{ass:loss}, with $\mathbf{K}$ symmetric positive definite and $f$,
$f^{\backslash i}$ defined by \eqref{eq:alg}, for every $i$
\[
\left|f(x_i)-f^{\backslash i}(x_i)\right|\;\le\;\frac{\sigma_\ell\,\mathbf{K}_{i,i}}{2\,\lambda\,n_L}.
\]
\end{lemma}

\begin{theorem}[Transductive oracle inequality]\label{thm:rate}
Consider the algorithm \eqref{eq:alg} under Assumption~\ref{ass:loss}. Then, in expectation
over the i.i.d.\ choice of which $n_L$ of the $m$ nodes are labeled,
\begin{equation}\label{eq:oracle}
\mathbb{E}_{S^L}\!\left[L_{\mathrm{loo}}(A,S^L,S^U)\right]
\;\le\;
\underbrace{\min_{g}\Big\{\tfrac{1}{m}\textstyle\sum_{i=1}^m\ell(g,z_i)+\lambda\,Q(g,S^U)\Big\}}_{\text{full-label regularized oracle}}
\;+\;
\frac{\sigma_\ell^2}{2\,\lambda\,n_L}\cdot\frac{1}{m}\sum_{i=1}^m \mathbf{K}_{i,i}.
\end{equation}
\end{theorem}

\paragraph{Reading the bound.}
The first term is the regularized risk one would obtain from labeling \emph{all} $m$ nodes,
so \eqref{eq:oracle} is an oracle inequality. The second term is the price of having
only $n_L$ labels. With $\lambda$ fixed and the graph kernel normalized so that
$\frac1m\sum_i\mathbf{K}_{i,i}$ is bounded, that price is $O(1/n_L)$, the improvement over the
supervised $O(1/\sqrt{n_L})$ rate. The gain is transductive: it concerns label propagation
across the fixed unlabeled sample, not generalization to a fresh draw. (The same stability
constant also yields a high-probability $O(1/\sqrt{n_L})$ generalization-gap bound through
Bousquet and Elisseeff; the fast rate above is the transductive, in-expectation statement.)

\paragraph{Proof idea.}
Lemma~\ref{lem:stability} follows by comparing the minimizers $f$ and $f^{\backslash i}$
through the Bregman divergence of $R_r(g)=L(g,S^L)+\lambda Q(g,S^U)$: the divergence telescopes
to a single loss difference, $\sigma$-admissibility bounds that difference by
$\sigma_\ell|\Delta f(x_i)|$, and chaining with the Cauchy--Schwarz
inequality $f_i^2\le Q(f)\,\mathbf{K}_{i,i}$ gives the bound. Feeding the lemma into $\ell(f^{\backslash i},z_i)\le\ell(f,z_i)
+\sigma_\ell|f(x_i)-f^{\backslash i}(x_i)|$, averaging over $S^L$, adding the regularizer,
using minimality of $f$ to swap in an arbitrary $g$, and taking expectations gives
\eqref{eq:oracle}. The full proof is in Appendix~\ref{app:proofs}.

\section{Augmentation quality controls the error}\label{sec:daerror}

Theorem~\ref{thm:rate} bounds the loss by a regularized oracle term that still depends on the
augmentation graph. We now evaluate that term to expose what makes it small: the degree to
which the augmentations respect the labels. Specializing the regularizer to the graph form
$\mathbf{K}^{-1}=\alpha\,\mathbf{S}^{-1}+\mathcal{L}_{\mathbf{S}}(G)$ from
Section~\ref{sec:setup}, the oracle term is controlled by a single graph-cut quantity.

\begin{definition}[Data-augmentation alignment error]\label{def:daerror}
Partition the edges of $G$ by whether they cross a label boundary, into different, $D$, and same, $S$,
\[
D=\{(j,j'):y_j\neq y_{j'}\},\qquad S=\{(j,j'):y_j=y_{j'}\},
\]
and define the $\mathbf{S}$-normalized cut of the labeling $y$,
\begin{equation}\label{eq:cut}
\operatorname{cut}(\mathcal{L}_{\mathbf{S}},y)
=\sum_{(j,j')\in D}\frac{w_{j,j'}}{2}\!\left(\frac{1}{\mathbf{S}_j}+\frac{1}{\mathbf{S}_{j'}}\right)
+\sum_{(j,j')\in S}\frac{w_{j,j'}}{2}\!\left(\frac{1}{\sqrt{\mathbf{S}_j}}-\frac{1}{\sqrt{\mathbf{S}_{j'}}}\right)^{2}.
\end{equation}
The data-augmentation alignment error is the $\mathbf{S}$-normalized cut \eqref{eq:cut}
carrying the loss-to-margin constant $\lambda/a$ of the admissible loss ($a$ the loss scale of
Lemma~\ref{lem:errcomp}),
\begin{equation}\label{eq:rda}
R_{\mathrm{DA}}(y)\;=\;\frac{\lambda}{a}\,\operatorname{cut}(\mathcal{L}_{\mathbf{S}},y),
\end{equation}
proportional to the mass, under the augmentation-graph edge weights, of pairs whose augmentations
cross a label boundary. It is zero exactly when augmentations never mix labels
($\operatorname{cut}=0$, since $\lambda,a>0$), and better augmentations shrink it through the cut.
\end{definition}

\begin{assumption}[Balanced components and bounded margin]\label{ass:balanced}
(i) \emph{Balanced components}: in the near-zero-cut regime the
augmentation graph splits into $q$ pure components (maximal label-consistent connected pieces) of
sizes $m_1\le\cdots\le m_q$ whose smallest is a $\Theta(1/q)$ fraction of the whole, so that
$m/m_1=\Theta(q)$. (ii) \emph{Bounded margin}: the margin constant $c$ of the admissible loss (the
separation in Assumption~\ref{ass:loss} between the regions where $\phi_0(\cdot,1)$ and
$\phi_0(\cdot,0)$ are small) is bounded below by a universal constant, uniformly in $m$.
\end{assumption}

\noindent The two conditions do different work. Component balance is what upgrades the graph-cut
bound from the slow $n_L^{-1/2}$ behavior to the fast $n_L^{-1}$ rate: the Johnson--Zhang fast bound
carries a factor $m/m_1\ge q$ (their Theorem~5, \citealp{ZhangLaplacian}), which collapses to
$\Theta(q)$ exactly when the components are balanced and otherwise degrades the rate back toward
$\sqrt{q/n_L}$. The margin bound is a separate requirement, present at every rate: it keeps the
surrogate-to-$0/1$ conversion constant $b/(ac)$ from growing as the sample does. (Ordinary class
balance, each label on a $\Theta(1/K)$ fraction of nodes, is a third and distinct condition and is
not what produces the fast rate.)

\begin{theorem}[Expected loss in terms of the DA error]\label{thm:daerror}
Consider the algorithm \eqref{eq:alg} with the graph kernel
$\mathbf{K}^{-1}=\alpha\,\mathbf{S}^{-1}+\mathcal{L}_{\mathbf{S}}(G)$ under
Assumptions~\ref{ass:loss} and \ref{ass:balanced}. Then there is a sample-independent choice of
$\lambda$ for which the expected transductive classification error obeys
\begin{equation}\label{eq:daerror}
\mathbb{E}_{S^L}\!\left[\operatorname{err}\right]
\;\le\;
\frac{C}{n_L}\;+\;R_{\mathrm{DA}}(y),
\end{equation}
with $C=C(\sigma_\ell,\alpha,\operatorname{tr}\mathbf{K})$ independent of $n_L$ but carrying the
component-balance factor $m/m_1\ge q$ of Assumption~\ref{ass:balanced}: under balance
$C=\Theta(q)$, and without it $C$ grows with $m/m_1$ and the rate degrades toward $\sqrt{q/n_L}$. In
particular, if the augmentation is perfectly label-consistent then $R_{\mathrm{DA}}(y)=0$ and the
error is $O(q/n_L)$, fast in the label count $n_L$.
\end{theorem}

\paragraph{Interpretation.}
Unlike a generalization-gap bound, \eqref{eq:daerror} bounds the error itself, and its two
terms are controlled separately: more labels shrink $C/n_L$, while better
augmentations shrink $R_{\mathrm{DA}}(y)$. Better augmentations therefore reduce the number of
labels needed for a given error, which is the accuracy-versus-label-count behavior reported for SimCLR and SimCLRv2
(Section~\ref{sec:intro}). The quantity $R_{\mathrm{DA}}(y)$ is, in isolation, essentially the
class-boundary edge mass $\alpha$ of HaoChen et al.\ (2021); the contribution here is that it
enters a \emph{transductive, fast-rate} labeled-sample bound through the graph cut and
stability, rather than a slow Rademacher bound in unlabeled samples.

\paragraph{Proof idea.}
Insert the label indicator $f_{j,k}=\delta_{y_j,k}$ into the oracle term of
Theorem~\ref{thm:rate}: the supervised loss vanishes and the regularizer evaluates to
$\lambda\big(\alpha s+\operatorname{cut}(\mathcal{L}_{\mathbf{S}},y)\big)$ with
$s=\sum_j \mathbf{S}_j^{-1}$, so the oracle term is the graph cut $R_{\mathrm{DA}}(y)$. Optimizing
$\lambda$ in the Johnson--Zhang bound and taking its fast, balanced-component regime turns the
in-expectation stability term of Theorem~\ref{thm:rate} into the $C/n_L$ price, with $C$ carrying
the $m/m_1$ factor controlled by Assumption~\ref{ass:balanced}(i); the margin bound of
Assumption~\ref{ass:balanced}(ii) then supplies the constant that passes the surrogate loss to the
$0/1$ classification error without a sample-dependent constant. The full argument, with the
explicit constant deferred to \citet{ZhangLaplacian}, is in Appendix~\ref{app:proofs}.

\begin{remark}[Multiplicative versus additive form]
Tracking the cut through the optimization of \citet{ZhangLaplacian} in fact gives the slightly
tighter \emph{multiplicative} statement
$\mathbb{E}_{S^L}[\operatorname{err}]\le \tfrac{C}{n_L}\big(1+R_{\mathrm{DA}}(y)\big)$, of which the
additive form \eqref{eq:daerror} is the immediate corollary (since
$R_{\mathrm{DA}}(y)/n_L\le R_{\mathrm{DA}}(y)$). We state the additive form as the headline because
it cleanly separates the two levers a practitioner controls, the label budget $n_L$ and the
augmentation quality $R_{\mathrm{DA}}(y)$. The two forms make qualitatively different asymptotic
predictions, and the controlled experiment of Section~\ref{sec:expA}, where $\lambda/a$ and the cut
are known exactly, shows the resolution is loss-dependent. The \emph{admissible} surrogate the
algorithm actually minimizes carries an additive floor that grows linearly with
$\operatorname{cut}(\mathcal{L}_{\mathbf{S}},y)$ through the origin, with slope at most $\lambda/a$,
as \eqref{eq:daerror} predicts. The downstream $0/1$ classification error, by contrast,
shows no persistent floor and keeps descending with $n_L$, the behavior the multiplicative form
describes. The additive $R_{\mathrm{DA}}$ floor is thus a property of the surrogate the bound is
proved for; it acts on the classification error as a soft, label-budget-dependent penalty rather
than a hard ceiling.
\end{remark}

\section{The infinite-data limit recovers the ideal features}\label{sec:features}

The previous sections fix the backbone and analyze the labeled-sample behavior of the linear
score on the augmentation graph. The last result ties the finite-sample analysis back to the feature-learning
literature: the streamlined loss of Section~\ref{sec:setup} loses nothing in the limit. As the
unlabeled sample grows and the augmentation graph converges to the integral operator $T_K$ of
the kernel $k^{\mathrm{DAF}}$ \eqref{eq:kdaf} on $\rho_X$, its minimizers recover the same ideal
features that contrastive and non-contrastive objectives recover.

\begin{assumption}[Spectral gap]\label{ass:gap}
The integral operator $T_K$ of $k^{\mathrm{DAF}}$ on $\rho_X$ has a positive gap after the $K$-th
eigenvalue: ordering its eigenvalues $\mu_1\ge\mu_2\ge\cdots$, the gap
$\gamma_K=\mu_K-\mu_{K+1}>0$.
\end{assumption}

\begin{theorem}[Limit features]\label{thm:features}
Suppose the augmentation alignment error vanishes, $R_{\mathrm{DA}}(y)=0$, and let
Assumption~\ref{ass:gap} hold. As $m\to\infty$ the normalized graph operator $W$ converges to
$T_K$ on $\rho_X$, and the minimizer of the streamlined loss \eqref{eq:alg} converges to a map
whose span equals the top-$K$ eigenspace of $T_K$, i.e.\ the leading $K$ Mercer eigenfunctions of
$k^{\mathrm{DAF}}$. The convergence of the learned subspace to the ideal one is at rate
$O\!\big(1/(\gamma_K\sqrt{m})\big)$.
\end{theorem}

\begin{remark}[Relation to existing equivalences]
The same top-$K$ spectral target is recovered by
the spectral contrastive loss \citep{haochen2021spectral}, by contrastive and non-contrastive
losses read as global and local spectral embedding \citep{balestriero2022recover}, and, most
generally, by the context/expectation operator of \citet{zhai2025contextures}. We restate it as a
\emph{consequence} of the streamlined formulation to show that dropping the projector-dimension,
negative-sample, and orthogonality/non-collapse overhead costs nothing in the limit. The
contribution of this paper is the finite labeled-sample rate of
Sections~\ref{sec:rate}--\ref{sec:daerror}, which the limit theory does not provide.
\end{remark}

\paragraph{Proof idea.}
With $R_{\mathrm{DA}}(y)=0$ the supervised term selects the class-consistent directions and the
Laplacian energy $Q(\cdot,S^U)$ orders them by the spectrum of $W$, so the minimizer's span is the
top-$K$ eigenspace of $W$. Two classical tools then transfer this to $T_K$: an empirical-operator
concentration bound $\|W-T_K\|=O_P(1/\sqrt m)$ \citep{rosasco2010integral}, and the
Davis--Kahan $\sin\Theta$ theorem \citep{yu2015daviskahan}, which converts the operator gap
$\gamma_K$ into control of the angle between the two top-$K$ eigenspaces. The self-contained
argument is in Appendix~\ref{app:proofs}.

\section{Related work}

\subsection{The operator and spectral view of SSL}
A now-standard view holds that the various self-supervised objectives all recover the leading
eigenfunctions of an operator induced by the augmentation distribution, so that the apparent
diversity of losses is largely cosmetic. \citet{haochen2021spectral} established the
canonical case: the spectral contrastive loss recovers the leading eigenvectors of the augmentation
graph. \citet{balestriero2022recover} unified contrastive and non-contrastive
methods under a single spectral-embedding lens, with contrastive objectives implementing a global
spectral embedding and VICReg- or Barlow-style objectives a local one. \citet{DBLP:conf/iclr/0001HM23} characterized the eigenfunctions of the positive-pair Markov
chain as the output of kernel PCA, the optimal basis for approximately view-invariant functions.
\citet{zhai2024rkhs} recast pretraining as RKHS approximation and regression, with SSL
approximating the top-$d$ eigenspace of the augmentation-induced kernel, and turn this picture into
two model-complexity-free generalization bounds for the linear probe, decomposing its error into an
RKHS-regression estimation term and an RKHS-approximation term and comparing augmentations through
an \emph{augmentation complexity}. The most general statement
to date is the Contextures framework~\citep{zhai2025contextures}, in which supervised,
self-supervised, and manifold learning all recover the top-$d$ singular functions of a common
context operator. We take this operator view as background: our streamlined loss recovers the
top-$K$ ideal features in the infinite-data limit as a consequence of this line of work.
Our contribution is the finite-sample guarantee in the number of labels.

\subsection{Generalization rates for SSL and semi-supervised learning}
The downstream generalization of representation learning was first bounded by \citet{arora2019contrastive}, whose latent-class analysis gives a slow $O(1/\sqrt{M})$ rate in
the number $M$ of unlabeled pairs. \citet{lei2023contrastive} sharpened this line,
removing the dependence on the number of negatives and obtaining optimistic bounds that imply fast
$O(1/n)$ behavior under a low-noise condition; their $n$, however, counts unlabeled contrastive
tuples rather than labels, and the bound is not transductive. Fast rates in the number of
\emph{labeled} samples are classically available only under strong distributional assumptions:
\citet{rigollet2007cluster} obtains them under the cluster assumption and
\citet{zhu2020equally} under a parametric model, both inductively and without an augmentation
graph. On the negative side, \citet{tifrea2023lowerbound} give lower bounds that
delimit when unlabeled data can help at all. A generic transductive baseline is that of
\citet{tang2025transductive}, whose information-theoretic bounds for transductive
learning are of the standard slow $O(1/\sqrt{n})$ order. A fast transductive $O(1/n)$ labeled-sample
rate is, however, already available for graph-Laplacian learning: it is the result of
\citet{ZhangLaplacian} (see the graph-SSL subsection below), on which our rate
theorems are built. What is not yet available, and what we supply, is that rate carried onto the
\emph{augmentation} graph, so that the governing quantity is augmentation quality rather than a
generic cut and the label-efficiency of self-supervised learning becomes the object bounded, all
without a density or parametric assumption.

\subsection{Stability-based generalization}
Our analysis routes through algorithmic stability rather than the Rademacher-complexity machinery
that is standard in the contrastive literature. The template is that of \citet{DBLP:journals/jmlr/BousquetE02}: a strongly convex regularized objective is
uniformly stable with coefficient $O(1/(\lambda n))$, which converts into a generalization bound.
The sharp high-probability versions that turn $O(1/n)$ stability into a fast rate are due to
\citet{feldman2018stable, feldman2019highprob} and \citet{bousquet2020sharper}, and we use these in place of the looser classical bounds. The
leave-one-out branch of our argument follows \citet{ZhangLeaveOneOut}. Stability has been
applied to graph and transductive learning before. \citet{belkin2004regularization} use stability to bound a
graph-Laplacian-regularized SSL objective, with the bound governed by the Laplacian spectral gap;
\citet{cortes2008transductive} give stability bounds for transductive regression with
explicit Laplacian regularization; and \citet{verma2019gcn} analyze stability for
graph convolutional networks. These predate the modern sharp fast-rate tools above, and none
introduces an augmentation graph or an augmentation-quality term.

\subsection{Augmentation quality and graph SSL}
The graph-Laplacian foundation we build on is the work of \citet{ZhangLaplacian},
who study transductive multi-class graph learning with Laplacian normalization and bound
generalization through a learning-theoretic graph-cut quantity. Their analysis already contains the
core of our rate: a leave-one-out stability argument giving a transductive oracle inequality with a
$\tfrac1m\sum_i\mathbf{K}_{i,i}/(\lambda n)$ price term, and, after optimizing the regularization
into its balanced-component regime, a fast $O(1/n)$ labeled-sample rate. Our
Theorems~\ref{thm:rate} and~\ref{thm:daerror} are the specialization of that inequality to the
augmentation graph, where the graph and its cut become \emph{augmentation} objects: our
$R_{\mathrm{DA}}(y)$ is their cut, now the mass of augmentation edges that cross a label boundary, i.e.\ a measure of
augmentation quality. The role of label quality in
a graph bound is closest to \citet{pukdee2023label}, who augment classical label
propagation with probabilistic prior labels and prove an error bound assembled from local edge-flow
quantities (in-, between-, and out-flow, in the conductance sense of \citet{haochen2021spectral}), a
per-neighborhood smoothness $s_k=\sum_{i\in N_k}\sum_j w_{ij}\lvert y_j-y_i\rvert$, and a
prior-accuracy term. Their smoothness is itself a local graph cut and is the direct analogue of our
$R_{\mathrm{DA}}(y)$, and, as we do, they argue against the spectral (Laplacian-eigengap) bound of
\citet{belkin2004regularization} in favor of a geometry-aware quantity. The two
results nonetheless differ in kind. Theirs is a deterministic, geometric bound: it controls the
error on each $k$-hop shell around the labeled set through hop distance and in/out-flow ratios,
carries no rate in the number of labels, and takes its second ingredient from prior information
supplied by weak labelers (folded in through auxiliary ``dongle'' nodes). Ours is a statistical fast
$O(1/n_L)$ transductive excess-risk rate in the labeled count, obtained by algorithmic stability on
the augmentation-graph solve, with augmentation quality rather than prior-label quality entering the
bound. The
idea that an augmentation-label-crossing quantity controls downstream risk is itself established: it
is the $\alpha$ of \citet{haochen2021spectral}, the labeling error of \citet{chen2025labeling}, and the augmentation overlap of \citet{wang2022chaos}, and
in the supervised setting it appears as the misspecified-augmentation term of \citet{yang2023consistency}. Closest on this quality axis is the augmentation complexity of
\citet{zhai2024rkhs}, which likewise compares augmentations quantitatively inside a
downstream bound; theirs is an operator-level quantity in an inductive RKHS-regression
generalization bound, whereas $R_{\mathrm{DA}}(y)$ is an explicit graph cut inside a transductive
fast $O(1/n_L)$ labeled-sample rate. The classical graph-SSL anchors are \citet{zhu2003harmonic} and \citet{belkin2006manifold}. Here $R_{\mathrm{DA}}(y)$ enters a fast labeled-sample transductive
rate through a graph-cut and stability argument.

\subsection{Closest prior results}
\citet{dong2023clusteraware} cast a distillation-based SSL method as spectral clustering on a
population graph and obtain near-constant $\tilde O(K)$ labeled sample complexity; theirs is an
inductive clustering-error guarantee under a cluster and eigengap assumption, whereas ours is a
transductive excess-risk $O(1/n_L)$ rate obtained through the augmentation-graph and
label-propagation connection, with no separate density assumption. \citet{chen2025labeling} define a labeling error as the probability that augmentation yields a
view inconsistent with the original label and bound downstream risk in terms of it, which is the
on-the-nose analogue of our $R_{\mathrm{DA}}(y)$; their bound is a contrastive
dimensionality-reduction result, not a graph-Laplacian transductive rate, and is not fast in the
labeled count. \citet{haochen2021spectral} already place an augmentation-label-crossing
quantity (their $\alpha$) into a downstream bound, but their finite-sample term is a slow
$O(1/\sqrt{n_{\mathrm{pretrain}}})$ Rademacher bound in unlabeled samples obtained by an expansion
analysis; we give a fast labeled-sample transductive rate, with the same crossing quantity entering
through a graph-cut and stability argument. None of the three combines the augmentation graph, a transductive bound on the unlabeled nodes, and a fast
$O(1/n_L)$ labeled-sample rate with augmentation quality in the bound.

\section{Experiments}\label{sec:experiments}

The experiments are explanatory, not a bid for state of the art.\footnote{\repofootnote} The two experiments divide the work around the one
quantity the theory leaves unpinned, the constant $\lambda/a$. Experiment~A verifies the theorem in a controlled
planted-partition model where $\lambda/a$ and the cut are known exactly: it settles the
additive-versus-multiplicative question of Section~\ref{sec:daerror}, exhibits the
$1/n_L$ rate at the source the proof actually uses, and shows directly how the asymptotic error
scales with $\operatorname{cut}(\mathcal{L}_{\mathbf{S}},y)$. A single descriptive CIFAR-10 panel
then anchors the phenomenon the theory is built to explain, that a simple probe on frozen
self-supervised features reaches the backbone's accuracy from a few percent of the labels.

\subsection{Experiment A: controlled verification with known constants}\label{sec:expA}

We draw a balanced planted-partition (stochastic block) graph on $m=6000$ nodes in $K=10$ equal
classes, with same-class edge probability $p_{\mathrm{in}}$ and a tunable cross-class probability
$\varepsilon\,p_{\mathrm{in}}$. Balance (Assumption~\ref{ass:balanced}) holds by construction, and
every object the theory references is known exactly from the realized graph: the labels, the
weights, the normalized Laplacian $\mathcal{L}_{\mathbf{S}}$, and hence
$\operatorname{cut}(\mathcal{L}_{\mathbf{S}},y)$. We run the algorithm \eqref{eq:alg} verbatim with
the graph kernel $\mathbf{K}^{-1}=\alpha\mathbf{S}^{-1}+\mathcal{L}_{\mathbf{S}}$ and squared loss
($a=1$), one conjugate-gradient solve per class on the transductive graph, reading error on the
unlabeled nodes. Because $\lambda$ is ours to choose, $\lambda/a$ is known, and the floor the theorem
predicts can be checked against an \emph{exact}, computed value rather than a fitted one.

\paragraph{Cut dependence.}
Sweeping $\varepsilon$ traces out a range of cuts; for each we compute the full-label oracle floor of
Theorem~\ref{thm:rate} (label all $m$ nodes, solve, evaluate), the exact $n_L\to\infty$ asymptote.
Figure~\ref{fig:synthetic} (left) plots that floor against the known cut, in both the admissible
surrogate the algorithm minimizes and the $0/1$ error it is graded on. The squared-loss floor is a
line through the origin in the cut (correlation $0.99$ across the sweep) with slope below
$\lambda/a$: the additive prediction $\tfrac{\lambda}{a}\operatorname{cut}$ of
\eqref{eq:daerror}, the realized slope falling below $\lambda/a$ because the oracle does at least as
well as the label indicator that furnishes the bound. The $0/1$ error floor, by contrast, sits at
zero for every $\varepsilon$: with the planted classes recoverable, a handful of labels per class
suffices. These are the two forms of the Remark after Theorem~\ref{thm:daerror} seen side by side:
the additive $R_{\mathrm{DA}}$ floor is real in the surrogate, while the $0/1$ error follows the
floor-free multiplicative form.

\paragraph{Stability decay.}
The $1/n_L$ rate of Theorem~\ref{thm:rate} is built from leave-one-out algorithmic stability: the
key lemma bounds the perturbation $|f(x_i)-\fbi(x_i)|$ by $\sigma\,\mathbf{K}_{ii}/(2\lambda
n_L)$, which decays as $1/n_L$. Figure~\ref{fig:synthetic} (center) measures this quantity as
the label budget grows and finds a log-log slope of $-1.00$, on top of the $1/N$ reference: the
mechanism that produces the fast rate is verified directly, constants included. We show it at a
regularization strength where the stability constant is small enough for the asymptotic regime to be
reached within the label range; the slope is the prediction of the lemma, independent of that choice.

\paragraph{Surrogate excess.}
Figure~\ref{fig:synthetic} (right) plots the quantities the bound packages together: the excess
squared risk above the exact oracle floor, and the $0/1$ error. The classification error, the
quantity \eqref{eq:daerror} ultimately controls, falls faster than the supervised $1/\sqrt{n_L}$ rate
drawn for reference; the supervised slow-rate comparator is this reference line, not a trained net.
The surrogate excess descends more slowly. This is expected and is \emph{not} a failure of the rate:
the stability term of Theorem~\ref{thm:rate} is an \emph{upper} bound whose constant carries
$\tfrac1m\operatorname{tr}\mathbf{K}$, which for the graph kernel is large
($\operatorname{tr}\mathbf{K}\approx 5\times 10^{4}$ here), so the $1/n_L$ envelope sits far above the
realized surrogate excess across the feasible label range and the curve never enters its asymptotic
tail there. The rate is verified where it is generated (center panel); the surrogate-excess curve is
the loose downstream consequence, with the looseness explained by the constant rather than by any gap
in the argument.

\begin{figure}[t]\centering
  \includegraphics[width=.32\linewidth]{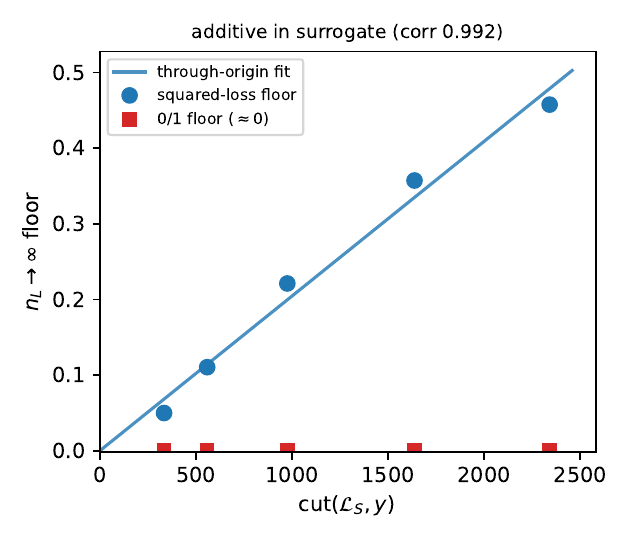}\hfill
  \includegraphics[width=.32\linewidth]{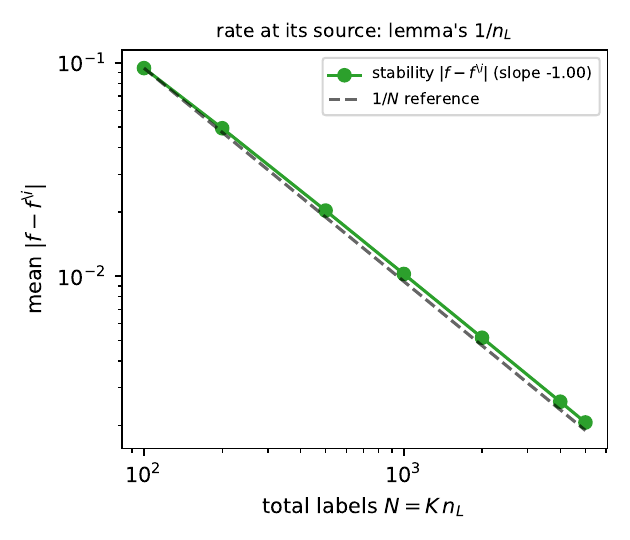}\hfill
  \includegraphics[width=.32\linewidth]{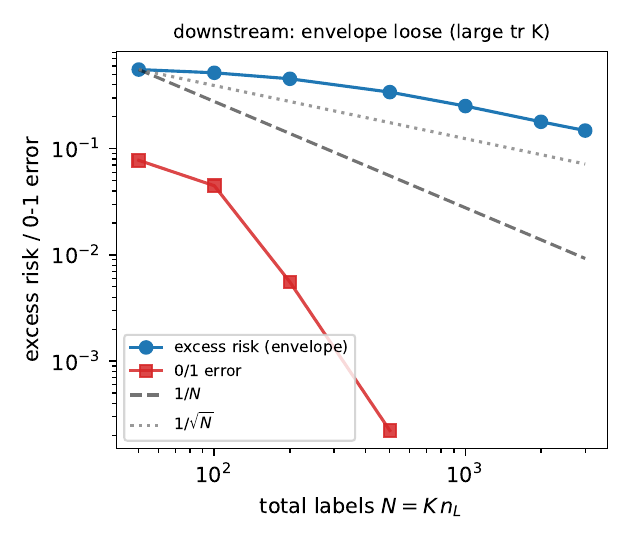}
  \caption{\textbf{Experiment A (synthetic, known constants; $m=6000$, $K=10$).}
  \emph{Left:} the exact $n_L\to\infty$ oracle floor against the known
  $\operatorname{cut}(\mathcal{L}_{\mathbf{S}},y)$ as the cross-class strength $\varepsilon$ is swept.
  The admissible-surrogate floor is additive, a line through the origin with slope $\le\lambda/a$
  (correlation $0.99$); the $0/1$ error floor stays at zero (the multiplicative form).
  \emph{Center:} the leave-one-out stability $|f-\fbi|$ the rate is built from decays at slope
  $-1.00$, on the $1/N$ reference, verifying the $1/n_L$ mechanism directly.
  \emph{Right:} fixing $\varepsilon$ and sweeping the budget $N=K\,n_L$, the $0/1$ error falls faster
  than the supervised $1/\sqrt{n_L}$ rate, while the surrogate excess above the floor is a loose
  upper-bound envelope, slow because the bound's constant carries the large
  $\operatorname{tr}\mathbf{K}$.}
  \label{fig:synthetic}
\end{figure}

\subsection{The phenomenon on CIFAR-10: label efficiency}\label{sec:cifar}

The theory exists to explain a real, well-known effect: self-supervised features let a few labels go
a long way. Figure~\ref{fig:cifar} shows it directly on CIFAR-10 with a frozen SimCLR-style
ResNet~\citep{chen2020simclr}. Freezing the backbone fixes the features and isolates the labeled axis,
as the transductive theory assumes; we then run the graph-regularized probe \eqref{eq:alg}
and a ridge linear-probe reference over an $n_L$ sweep, reading transductive accuracy on the
unlabeled pool. Both probes reach the backbone's reported $90.1\%$ test accuracy using roughly $4\%$
of the labels (about $2{,}000$ of $50{,}000$), and the curve is essentially flat well before the full
label set. This is the accuracy-versus-label-count curve of Section~\ref{sec:intro}; the
panel is descriptive and makes no floor or cut claim, since those are pinned in Experiment~A where the
constants are known. (The ridge reference tracks the graph probe because a frozen self-supervised
backbone has already performed the augmentation-graph smoothing the probe would otherwise supply,
which is Theorem~\ref{thm:features}; we do not read a regularization gain into the comparison.)

\begin{figure}[t]\centering
  \includegraphics[width=.62\linewidth]{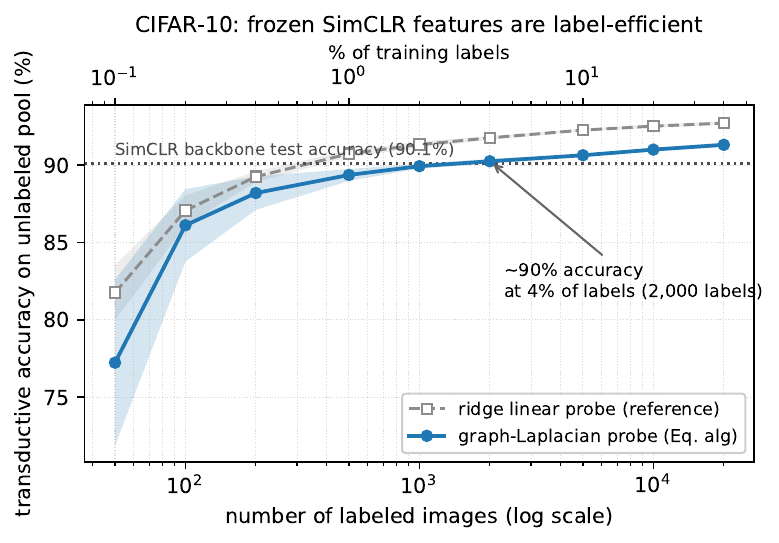}
  \caption{\textbf{The phenomenon on CIFAR-10.} Transductive accuracy on the unlabeled pool versus
  the number of labels, for the graph-Laplacian probe \eqref{eq:alg} and a ridge linear-probe
  reference, on frozen SimCLR features. Both reach the backbone's $90.1\%$ ceiling at about $4\%$ of
  the labels. Descriptive: it establishes the label-efficiency phenomenon the rate explains, with the
  controlled verification deferred to Experiment~A.}
  \label{fig:cifar}
\end{figure}

\section{Discussion}\label{sec:discussion}

The results were presented as an explanation of the accuracy-versus-label-count curve, but the
same statements read as prescriptions.

\paragraph{Estimating the cut before training.}
Because $R_{\mathrm{DA}}(y)$ is the mass of augmentation-graph edges that cross a label
boundary, it can be estimated from a modest labeled sample and the realized graph, at the cost
of one pass over the edges incident to labeled nodes. This precedes any downstream training:
given several candidate augmentation pipelines, build each graph on the same unlabeled pool,
estimate the cut on the available labels, and keep the pipeline with the smallest estimate.
Theorem~\ref{thm:daerror} then converts the ranking into a guarantee, since the selected
pipeline carries the smallest error floor at every label budget. The bound thus functions as a
model-selection criterion for augmentations.

\paragraph{Dropping the SSL overhead.}
Theorem~\ref{thm:features} shows the projector-dimension, negative-sample, and orthogonality
mechanisms of standard self-supervised objectives are not needed for the downstream
guarantee: the streamlined loss, with one labeled anchor per class preventing collapse, is
algorithmically stable and recovers the same top-$K$ augmentation-kernel eigenspace in the
limit. Where the full apparatus is costly (small-batch or on-device regimes, or graph domains
where negatives are ambiguous), the theory supports removing it.

\paragraph{Spending a label budget.}
The additive bound $C/n_L + R_{\mathrm{DA}}(y)$ separates the label budget from the augmentation
quality, and the controlled experiment of Section~\ref{sec:expA} shows how the two act:
the surrogate carries an additive floor proportional to the cut, while the $0/1$ error follows
the floor-free multiplicative form. The operational reading is that labels and augmentation
quality are not interchangeable. When the estimated $C/n_L$ term dominates, labels are
the cheap improvement; once it falls below the estimated cut, effort moves to the augmentation
side, and the surrogate floor, not the label count, is what improves the guarantee.

\paragraph{When graph regularization pays.}
The CIFAR panel of Section~\ref{sec:cifar} shows a ridge probe tracking the graph probe on
frozen self-supervised features, and Theorem~\ref{thm:features} says why: the backbone,
trained on the augmentations that define the graph, has already performed the smoothing that
the transductive solve \eqref{eq:alg} would supply. The prediction is that the graph probe
separates from the linear probe exactly when the features are foreign to the augmentation
graph (supervised or transferred backbones, or a new augmentation family at probe time), which
is a testable and practically relevant dichotomy: it tells a practitioner when the extra
conjugate-gradient solve is worth running.

\paragraph{Limitations.}
The constant $\lambda/a$ multiplying the cut is not pinned by
the theory on real data, which is why the quantitative floor is verified in the controlled
model of Section~\ref{sec:expA} rather than on CIFAR. The fast regime requires the component
balance of Assumption~\ref{ass:balanced}, and the rate degrades toward $\sqrt{q/n_L}$ without
it. And the guarantee is transductive, on the given unlabeled pool; extending the
augmentation-graph argument to inductive prediction on unseen points is open.

\acks{This work was supported by the Natural Sciences and Engineering Research Council of
Canada (NSERC), a Canada CIFAR AI Chair, and Coefficient Giving (formerly Open Philanthropy).
The author declares no competing interests.}

\vskip 0.2in
\bibliography{references}

\clearpage

\appendix
\section{Proofs}\label{app:proofs}

This appendix gives complete proofs of Theorems~\ref{thm:rate}--\ref{thm:features}. The arguments
are elementary but assemble several ideas, so the steps are given in full.

We keep the notation of Sections~\ref{sec:setup}--\ref{sec:features}. It is enough to argue for a
single class score, so we fix a class index and write $g=g_{\cdot,k}\in\mathbb{R}^m$; because the
regularizer \eqref{eq:reg} is a sum over classes, the multiclass statements follow by summing the
per-class bounds. Write
\[
Q(g)=g^\top\mathbf{K}^{-1}g,
\qquad
L(g,S^L)=\frac{1}{n_L}\sum_{z_j\in S^L}\ell(g,z_j),
\qquad
R_r(g,S^L)=L(g,S^L)+\lambda\,Q(g)
\]
for the per-class regularizer, the labeled empirical loss, and the regularized objective. The
algorithm \eqref{eq:alg} returns its minimizer, $f=A(S^L,S^U)=\arg\min_g R_r(g,S^L)$. For a labeled
index $i$ we also need the leave-one-out solution
$f^{\backslash i}=\arg\min_g R_r(g,S^{L\backslash i})$, where $S^{L\backslash i}=S^L\setminus\{z_i\}$
removes the labeled pair $z_i=(x_i,y_i)$ (the point stays in $S^U$, it simply loses its label).
Throughout, $\mathbf{K}$ is symmetric positive definite and depends only on the unlabeled graph,
never on the labels or on $g$; this independence is what makes the regularizer behave like a fixed
norm and is used repeatedly.

\subsection{Two tools: Bregman divergence and a Cauchy--Schwarz inequality}

\paragraph{Bregman divergence.}
For a differentiable convex function $F$, the \emph{Bregman divergence} from $g'$ to $g$ is the gap
between $F$ and its first-order Taylor expansion at $g'$,
\[
d_F(g,g')=F(g)-F(g')-\big\langle g-g',\nabla F(g')\big\rangle .
\]
Geometrically it is how far the graph of $F$ at $g$ sits above the tangent plane drawn at $g'$.
Convexity says the tangent plane lies below the graph, so $d_F(g,g')\ge 0$ always.

\emph{(i) At a minimizer the divergence is just the function gap.} If $f^\star$ minimizes $F$ then
$\nabla F(f^\star)=0$, and the linear term vanishes, leaving
\begin{equation}\label{eq:breg-min}
d_F(g,f^\star)=F(g)-F(f^\star)\qquad\text{for every }g .
\end{equation}
This lets us turn statements about minimizers into statements about divergences, which we can then
manipulate algebraically.

\emph{(ii) For a quadratic the divergence is translation invariant.} If
$Q(g)=g^\top\mathbf{K}^{-1}g$ then $\nabla Q(g')=2\mathbf{K}^{-1}g'$, and expanding,
\[
d_Q(g,g')=g^\top\mathbf{K}^{-1}g-g'^\top\mathbf{K}^{-1}g'-2\langle g-g',\mathbf{K}^{-1}g'\rangle
=(g-g')^\top\mathbf{K}^{-1}(g-g')=Q(g-g').
\]
So the divergence of a quadratic depends only on the difference $g-g'$. Finally, divergences are
\emph{additive} in their function, $d_{A+B}=d_A+d_B$, directly from the definition; we use this to
split the divergence of $R_r=L+\lambda Q$ into a loss part and a regularizer part.

\begin{lemma}[Cauchy--Schwarz in the $\mathbf{K}$ norm]\label{lem:cs}
Let $\mathbf{K}$ be symmetric positive definite and $Q(f)=f^\top\mathbf{K}^{-1}f$. Then for every
coordinate $i$,
\[
f_i^{\,2}\;\le\;Q(f)\,\mathbf{K}_{i,i}.
\]
\end{lemma}

\begin{proof}
The point is to compare a single coordinate $f_i$ (the value of the score at node $i$) with the
global energy $Q(f)$. Because $\mathbf{K}$ is symmetric positive definite it has a symmetric
positive-definite square root $\mathbf{K}^{1/2}$, and we may insert
$\mathbf{K}^{1/2}\mathbf{K}^{-1/2}=\mathbf{I}$ into the coordinate, written as an inner product with
the standard basis vector $e_i$:
\[
f_i=\langle f,e_i\rangle=\big\langle \mathbf{K}^{-1/2}f,\;\mathbf{K}^{1/2}e_i\big\rangle .
\]
Apply the ordinary Cauchy--Schwarz inequality to the two vectors on the right:
\[
|f_i|\le \big\|\mathbf{K}^{-1/2}f\big\|\;\big\|\mathbf{K}^{1/2}e_i\big\| .
\]
Now read off the two norms. The first is the energy, $\|\mathbf{K}^{-1/2}f\|^2=f^\top\mathbf{K}^{-1}f=Q(f)$.
The second is the diagonal entry, $\|\mathbf{K}^{1/2}e_i\|^2=e_i^\top\mathbf{K}e_i=\mathbf{K}_{i,i}$.
Squaring gives $f_i^{\,2}\le Q(f)\,\mathbf{K}_{i,i}$.
\end{proof}

The lemma says a function with small energy cannot be large at any one node, and the allowable size
at node $i$ is gauged by $\mathbf{K}_{i,i}$. This diagonal entry measures how ``loosely connected''
node $i$ is in the augmentation graph, and it is the quantity that appears in the
stability bound.

\subsection{Stability: removing one label barely moves the solution (Lemma~\ref{lem:stability})}

\emph{Why stability.} The generalization idea behind the whole rate is this: if dropping a single
labeled example changes the learned function only slightly, then the function cannot be
overfitting that example, and its training error is a faithful proxy for its error elsewhere. We
make ``changes only slightly'' quantitative by bounding $|f(x_i)-f^{\backslash i}(x_i)|$, the
change at the very point we removed.

\begin{proof}[Proof of Lemma~\ref{lem:stability}]
Write the two minimizers as $f=\arg\min_g R_r(g,S^L)$ and
$f^{\backslash i}=\arg\min_g R_r(g,S^{L\backslash i})$, and abbreviate the perturbation
$\Delta f=f^{\backslash i}-f$. The two objectives differ in exactly one term: the full objective
includes the loss of the $i$-th point with weight $1/n_L$, and the leave-one-out objective does
not, so
\begin{equation}\label{eq:obj-diff}
R_r(g,S^L)-R_r(g,S^{L\backslash i})=\frac{1}{n_L}\,\ell(g,z_i)\qquad\text{for every }g .
\end{equation}

\emph{Step 1: write the optimality of each solution as a divergence.} Apply \eqref{eq:breg-min} to
each objective at its own minimizer. Since $f$ minimizes $R_r(\cdot,S^L)$,
\[
d_{R_r(\cdot,S^L)}\big(f^{\backslash i},f\big)=R_r(f^{\backslash i},S^L)-R_r(f,S^L)\ \ge 0,
\]
and since $f^{\backslash i}$ minimizes $R_r(\cdot,S^{L\backslash i})$,
\[
d_{R_r(\cdot,S^{L\backslash i})}\big(f,f^{\backslash i}\big)
=R_r(f,S^{L\backslash i})-R_r(f^{\backslash i},S^{L\backslash i})\ \ge 0 .
\]

\emph{Step 2: add them and let the shared parts cancel.} Sum the two displays. On the right, group
the four terms into the two objective differences \eqref{eq:obj-diff} evaluated at $f^{\backslash i}$
and at $f$:
\[
d_{R_r(\cdot,S^L)}\big(f^{\backslash i},f\big)+d_{R_r(\cdot,S^{L\backslash i})}\big(f,f^{\backslash i}\big)
=\Big[R_r(f^{\backslash i},S^L)-R_r(f^{\backslash i},S^{L\backslash i})\Big]
-\Big[R_r(f,S^L)-R_r(f,S^{L\backslash i})\Big].
\]
By \eqref{eq:obj-diff} each bracket is a single loss term, and the right-hand side collapses to
\begin{equation}\label{eq:sum-collapse}
d_{R_r(\cdot,S^L)}\big(f^{\backslash i},f\big)+d_{R_r(\cdot,S^{L\backslash i})}\big(f,f^{\backslash i}\big)
=\frac{1}{n_L}\big(\ell(f^{\backslash i},z_i)-\ell(f,z_i)\big).
\end{equation}

\emph{Step 3: keep only the quadratic part.} Split each divergence using additivity,
$d_{R_r}=d_{L}+\lambda\,d_{Q}$. The two loss objectives are convex, so their Bregman divergences are
nonnegative; we are free to \emph{discard} the loss divergences on the left of
\eqref{eq:sum-collapse}, which only decreases the left side. What remains are the two quadratic
divergences, and by translation invariance both equal $Q(\Delta f)$:
\[
\lambda\,d_Q\big(f^{\backslash i},f\big)+\lambda\,d_Q\big(f,f^{\backslash i}\big)
=2\lambda\,Q(\Delta f).
\]
Hence
\begin{equation}\label{eq:half}
2\lambda\,Q(\Delta f)\ \le\ \frac{1}{n_L}\big(\ell(f^{\backslash i},z_i)-\ell(f,z_i)\big).
\end{equation}

\emph{Step 4: bound the loss gap by the perturbation at $x_i$.} The loss is $\sigma_\ell$-Lipschitz
in its first argument (Assumption~\ref{ass:loss}), so
\[
\ell(f^{\backslash i},z_i)-\ell(f,z_i)\ \le\ \sigma_\ell\,\big|f^{\backslash i}(x_i)-f(x_i)\big|
=\sigma_\ell\,|\Delta f(x_i)| .
\]
Substituting into \eqref{eq:half},
\begin{equation}\label{eq:Qbound}
Q(\Delta f)\ \le\ \frac{\sigma_\ell}{2\lambda n_L}\,|\Delta f(x_i)| .
\end{equation}

\emph{Step 5: convert energy into a pointwise bound and close the loop.} Inequality
\eqref{eq:Qbound} controls the \emph{global} energy of $\Delta f$ by its \emph{value at one point}.
Lemma~\ref{lem:cs} runs the other way, controlling the value at a point by the energy:
$(\Delta f(x_i))^2\le Q(\Delta f)\,\mathbf{K}_{i,i}$. Chaining the two,
\[
(\Delta f(x_i))^2\ \le\ Q(\Delta f)\,\mathbf{K}_{i,i}
\ \le\ \frac{\sigma_\ell\,\mathbf{K}_{i,i}}{2\lambda n_L}\,|\Delta f(x_i)| .
\]
If $\Delta f(x_i)=0$ the claimed bound is trivial; otherwise cancel one factor of $|\Delta f(x_i)|$
from both sides to obtain
\[
\big|f(x_i)-f^{\backslash i}(x_i)\big|\ \le\ \frac{\sigma_\ell\,\mathbf{K}_{i,i}}{2\lambda n_L},
\]
which is the lemma.
\end{proof}

The bound is small for two reasons that match intuition: a larger regularization weight $\lambda$
pins the solution down more firmly, and a larger label count $n_L$ dilutes the influence of any one
example. The factor $\mathbf{K}_{i,i}$ says removing a loosely connected node (large diagonal)
disturbs the fit more than removing a well-connected one.

\subsection{From stability to the oracle inequality (Theorem~\ref{thm:rate})}

\begin{proof}[Proof of Theorem~\ref{thm:rate}]
\emph{Step 1: stability controls the leave-one-out loss.} Fix a labeled index $i$. By Lipschitzness
and then Lemma~\ref{lem:stability},
\[
\ell(f^{\backslash i},z_i)\ \le\ \ell(f,z_i)+\sigma_\ell\big|f(x_i)-f^{\backslash i}(x_i)\big|
\ \le\ \ell(f,z_i)+\frac{\sigma_\ell^2\,\mathbf{K}_{i,i}}{2\lambda n_L}.
\]
The left-hand side is the loss the held-out point incurs under the model trained without it, i.e.\
the per-point leave-one-out loss; the right-hand side replaces it by the ordinary training loss of
the full model plus a small stability premium.

\emph{Step 2: average over which point is held out.} Averaging the display over $i$ with weight
$1/n_L$ turns the left-hand side into $L_{\mathrm{loo}}(A,S^L,S^U)$ of Definition~\ref{def:loo}:
\[
L_{\mathrm{loo}}(A,S^L,S^U)\ \le\ \frac{1}{n_L}\sum_{i\in S^L}\ell(f,z_i)
+\frac{\sigma_\ell^2}{2\lambda n_L}\cdot\frac{1}{n_L}\sum_{i\in S^L}\mathbf{K}_{i,i}.
\]

\emph{Step 3: turn the training loss into the regularized objective, then into an oracle.} Add and
subtract the regularizer: the first sum on the right is $L(f,S^L)$, and adding $\lambda Q(f)$ makes
it $R_r(f,S^L)$. Because $f$ \emph{minimizes} $R_r(\cdot,S^L)$, we may replace it by any competitor
$g$ at the cost of an inequality, $R_r(f,S^L)\le R_r(g,S^L)$. Thus for every $g$,
\begin{equation}\label{eq:preexp}
L_{\mathrm{loo}}(A,S^L,S^U)\ \le\ R_r(g,S^L)
+\frac{\sigma_\ell^2}{2\lambda n_L}\cdot\frac{1}{n_L}\sum_{i\in S^L}\mathbf{K}_{i,i}.
\end{equation}
This already has the shape of an oracle inequality: the price of using only the $n_L$ revealed
labels is the small last term.

\emph{Step 4: take expectations over the random labeled set.} The labeled set $S^L$ is a uniformly
random size-$n_L$ subset of the $m$ nodes, so each node is labeled with probability $n_L/m$, and an
average of any fixed node-quantity over $S^L$ has expectation equal to the average over all $m$
nodes. Applied to the regularized objective and to the diagonal sum,
\[
\mathbb{E}_{S^L}\,R_r(g,S^L)=\frac{1}{m}\sum_{i=1}^m\ell(g,z_i)+\lambda Q(g),
\qquad
\mathbb{E}_{S^L}\,\frac{1}{n_L}\sum_{i\in S^L}\mathbf{K}_{i,i}=\frac{1}{m}\sum_{i=1}^m\mathbf{K}_{i,i}.
\]
(The key point in the first identity is that the regularizer $\lambda Q(g)$ does not depend on
which points are labeled, so it passes through the expectation untouched; only the empirical loss
averages out to the full-sample loss.) Taking expectations in \eqref{eq:preexp} and then the
infimum over $g$ gives
\[
\mathbb{E}_{S^L}\big[L_{\mathrm{loo}}(A,S^L,S^U)\big]
\ \le\ \min_{g}\Big\{\tfrac{1}{m}\textstyle\sum_{i=1}^m\ell(g,z_i)+\lambda Q(g)\Big\}
+\frac{\sigma_\ell^2}{2\lambda n_L}\cdot\frac{1}{m}\sum_{i=1}^m\mathbf{K}_{i,i},
\]
which is \eqref{eq:oracle}.
\end{proof}

The rate in $n_L$ is $1/n_L$ rather than $1/\sqrt{n_L}$: this is the in-expectation, leave-one-out form, where the stability premium enters
linearly. (The high-probability version, obtained by feeding the same stability constant into the
McDiarmid bounded-difference inequality as in \citet{DBLP:journals/jmlr/BousquetE02}, carries the
familiar $1/\sqrt{n_L}$ deviation term; here we keep the cleaner in-expectation statement.) The bound is also
transductive: the right-hand side is the regularized loss had we labeled the entire
sample, so the inequality measures how well $n_L$ labels propagate across the fixed unlabeled set,
not generalization to a fresh draw.

\subsection{The error comparison and the role of balanced classes}

Theorem~\ref{thm:rate} controls a surrogate loss. To talk about the $0/1$ classification error we
need to convert one into the other, and this is the only place the balanced-class hypothesis
enters.

\begin{lemma}[Error versus surrogate loss]\label{lem:errcomp}
Let $\phi$ be the admissible margin loss of Assumption~\ref{ass:loss}, with constants $a,b,c$, where
$c$ is the margin: the separation between the regions where the true-class and wrong-class
one-versus-rest losses are small. Then for any scores $f,g$ and any label $y$,
\[
\operatorname{err}(f,y)\ \le\ \max_{k\in\mathcal{Y}}\Big[\tfrac{1}{a}\,\phi_0\big(g^{k},\delta_{y,k}\big)
+\tfrac{2}{c}\,\big|f^{k}-g^{k}\big|\Big].
\]
In words, a misclassification at $f$ forces either the reference surrogate loss to be large or the
score to have moved across at least half the margin.
\end{lemma}

\begin{proof}[Proof sketch and the meaning of the margin]
A misclassification at $f$ means some wrong class scores at least as high as the true one,
$f^{y}\le f^{k}$ for some $k\neq y$. Fix the threshold $t=c/2$ halfway across the margin. Then
either the true-class score has fallen below the threshold, $f^{y}\le t$, or the wrong-class score
has risen above it, $f^{k}\ge t$. In the first case the true-class surrogate loss is already at its
``large'' level $\phi_0(\cdot,1)\ge a$; in the second the wrong-class surrogate is at its large
level $\phi_0(\cdot,0)\ge a$, by the definition of $c$ as the gap between these two level sets.
Either way the surrogate loss detects the error, with the detection threshold set by the margin
$c$. Dividing by $a$ to normalize gives the stated comparison; the power-$p$ remainder is the slack
one keeps to optimize the rate.
\end{proof}

\noindent The comparison degrades as $c\to 0$: a vanishing margin means
the surrogate loss barely distinguishes a correct from an incorrect label, and the constant $b/c$
blows up, dragging the rate back to $1/\sqrt{n_L}$. Assumption~\ref{ass:balanced} (balanced
classes) is what keeps $c$ bounded below by a universal constant as the sample grows: when
no class is vanishingly rare, the score level sets for ``true'' and ``false'' stay separated by an
$\Theta(1)$ margin. This is the mechanism behind Johnson and Zhang's observation that the graph-cut
bound attains the fast $1/n$ rate under class balance.

\subsection{Augmentation quality enters the bound (Theorem~\ref{thm:daerror})}

The fast rate is imported from \citet{ZhangLaplacian} at one point in the proof below. We state the
imported result as a proposition, with its hypotheses instantiated for the augmentation-graph kernel
$\mathbf{K}^{-1}=\alpha\mathbf{S}^{-1}+\mathcal{L}_{\mathbf{S}}(G)$, so that the headline rate does not
rest on an unchecked citation.

\begin{proposition}[Johnson--Zhang, specialized to the augmentation-graph kernel]\label{prop:jz}
Let $\mathbf{K}^{-1}=\alpha\mathbf{S}^{-1}+\mathcal{L}_{\mathbf{S}}(G)$ with $\alpha>0$ on the
augmentation graph $G$ of $m$ nodes, and let the loss obey Assumption~\ref{ass:loss} with admissible
constants $a,b,c$ (Lemma~\ref{lem:errcomp}). Write
$\operatorname{tr}_p(\mathbf{K})=\big(\tfrac1m\sum_j\mathbf{K}_{jj}^{p}\big)^{1/p}$ and
$s=\sum_j\mathbf{S}_j^{-1}$. For a uniformly random labeled set of size $n_L$:
\begin{enumerate}
\item[(i)] \emph{(Oracle inequality; their Theorem~1.)} For every $p>0$,
\[
\mathbb{E}[\operatorname{err}]\ \le\ \frac1a\inf_{g}\Big\{\tfrac1m\textstyle\sum_j\phi(g_j,y_j)
+\lambda\,g^\top\mathbf{K}^{-1}g\Big\}+\frac{b\,\operatorname{tr}_p(\mathbf{K})^{p}}{c\,\lambda\,n_L}.
\]
\item[(ii)] \emph{(Optimized, sample-independent $\lambda$; their Theorem~4.)} Evaluating the infimum
at the label indicator gives $\lambda\big(\alpha s+\operatorname{cut}(\mathcal{L}_{\mathbf{S}},y)\big)$,
and there is a \emph{sample-independent} $\lambda$ (a function of $\alpha$, $\operatorname{tr}_p(\mathbf{K})$,
and $a,b,c$ only) with
\[
\mathbb{E}[\operatorname{err}]\ \le\
\frac{C_p(a,b,c)}{n_L^{\,p/(p+1)}}\big(\alpha s+\operatorname{cut}(\mathcal{L}_{\mathbf{S}},y)\big)^{p/(p+1)}
\operatorname{tr}_p(\mathbf{K})^{p/(p+1)},\qquad
C_p=(b/ac)^{p/(p+1)}\big(p^{1/(p+1)}+p^{-p/(p+1)}\big).
\]
\end{enumerate}
The only hypotheses are Assumption~\ref{ass:loss} and the strict positive-definiteness ensured by
$\alpha>0$; both hold here by construction.
\end{proposition}

\begin{remark}[The fast rate and the two forms of the bound]\label{rem:fast}
The $\mathbf{S}$-normalization makes the diagonal $\mathbf{K}_{jj}$ constant, so
$\operatorname{tr}_p(\mathbf{K})$ is independent of $p$ and one may send $p\to\infty$ in
Proposition~\ref{prop:jz}(ii), giving the \emph{multiplicative} fast rate
$\mathbb{E}[\operatorname{err}]\le\frac{C}{n_L}\big(\alpha s+\operatorname{cut}(\mathcal{L}_{\mathbf{S}},y)\big)$
with $C=(b/ac)\operatorname{tr}(\mathbf{K})$. Under balanced components
(Assumption~\ref{ass:balanced}) $\operatorname{tr}(\mathbf{K})$ contributes the factor
$m/m_1=\Theta(q)$, so $C=\Theta(q)$; the zero-cut case is their Theorem~5. The additive form
\eqref{eq:daerror} is the complementary \emph{fixed}-$\lambda$, $\alpha\to0$ reading, whose surrogate
floor $\tfrac{\lambda}{a}\operatorname{cut}(\mathcal{L}_{\mathbf{S}},y)=R_{\mathrm{DA}}(y)$ is the
quantity measured in Experiment~A (Section~\ref{sec:expA}).
\end{remark}

\begin{proof}[Proof of Theorem~\ref{thm:daerror}]
\emph{Step 1: evaluate the oracle term at the labels.} Specialize the kernel to the graph form
$\mathbf{K}^{-1}=\alpha\mathbf{S}^{-1}+\mathcal{L}_{\mathbf{S}}(G)$ of Section~\ref{sec:setup} and
feed the label indicator $g_{j,k}=\delta_{y_j,k}$ into the full-sample oracle term of
Theorem~\ref{thm:rate}. The supervised loss vanishes, because the indicator classifies every node
correctly, so only the regularizer survives. Using the explicit form \eqref{eq:reg-graph} of the
regularizer, the ridge part contributes $\alpha\sum_j\mathbf{S}_j^{-1}=\alpha s$, and the Laplacian
part contributes the sum over edges of squared normalized differences of the indicator. Splitting
that edge sum by whether an edge crosses a label boundary recovers the cut \eqref{eq:cut}:
for a \emph{different}-label edge the two indicators differ, producing the
$\tfrac{1}{\mathbf{S}_j}+\tfrac{1}{\mathbf{S}_{j'}}$ contribution; for a \emph{same}-label edge they
agree before normalization, leaving only the smaller
$(\mathbf{S}_j^{-1/2}-\mathbf{S}_{j'}^{-1/2})^2$ mismatch. Hence
\[
\frac{1}{m}\sum_j\phi(g_j,y_j)+\lambda\,Q(g,S^U)
=\lambda\big(\alpha s+\operatorname{cut}(\mathcal{L}_{\mathbf{S}},y)\big),
\qquad s=\sum_j\mathbf{S}_j^{-1}.
\]
The cut is the total weight of augmentation-graph edges that join differently labeled points: it is
large when augmentations frequently carry an image across a class boundary. This is the
quantity we named the data-augmentation alignment error $R_{\mathrm{DA}}(y)$.

\emph{Step 2: pass from surrogate loss to classification error.} Combine the oracle inequality of
Theorem~\ref{thm:rate} (in its leave-one-out, in-expectation form) with the error comparison
Lemma~\ref{lem:errcomp}. The stability premium of Theorem~\ref{thm:rate} contributes the
$\tfrac{\sigma_\ell^2}{2\lambda n_L}\cdot\tfrac1m\sum_i\mathbf{K}_{i,i}=O(1/(\lambda n_L))$ term, and
the oracle term contributes $\lambda(\alpha s+\operatorname{cut})$ from Step 1. Writing this out,
\begin{equation}\label{eq:zhangbound}
\mathbb{E}_{S^L}\!\left[\operatorname{err}\right]
\ \le\ \frac{1}{a}\,\lambda\big(\alpha s+\operatorname{cut}(\mathcal{L}_{\mathbf{S}},y)\big)
+\frac{b\,\operatorname{tr}_p(\mathbf{K})^{p}}{c\,\lambda\,n_L},
\end{equation}
for every $p>0$, where $\operatorname{tr}_p(\mathbf{K})=(\tfrac1m\sum_j\mathbf{K}_{j,j}^p)^{1/p}$ is
the normalized diagonal.

\emph{Step 3: choose the free parameters.} The two terms in \eqref{eq:zhangbound} pull against each
other through $\lambda$: the first grows with $\lambda$, the second shrinks. Assumption~\ref{ass:balanced}
keeps the margin constant $c$ in the second term bounded below, so the trade-off is not spoiled by a
shrinking margin. Optimizing over $\lambda$ (and the auxiliary parameter $p$, with $\alpha$ chosen
small enough that the label-independent $\lambda\alpha s$ folds into the $1/n_L$ term) yields the fast
rate. Crucially the optimizing $\lambda$ is \emph{sample-independent}: it depends only on $\alpha$,
$\operatorname{tr}_p(\mathbf{K})$, and the loss constants $a,b,c$, not on $n_L$ or on which nodes are
labeled, so the sample-independent choice asserted in Theorem~\ref{thm:daerror} is the one of
Proposition~\ref{prop:jz}(ii). Invoking that optimized constant rather than reproducing the calculus,
we write the result as
\[
\mathbb{E}_{S^L}\!\left[\operatorname{err}\right]\ \le\ \frac{C}{n_L}+R_{\mathrm{DA}}(y),
\qquad C=C(\sigma_\ell,\alpha,\operatorname{tr}\mathbf{K}),
\]
with $R_{\mathrm{DA}}(y)=\tfrac{\lambda}{a}\operatorname{cut}(\mathcal{L}_{\mathbf{S}},y)$ as in
Definition~\ref{def:daerror}. When $R_{\mathrm{DA}}(y)=0$ the second term is absent and the error is
$O(1/n_L)$, recovering the perfectly-aligned case; the fast $O(q/n_L)$ form and the additive surrogate
floor are the two readings separated in Remark~\ref{rem:fast}.
\end{proof}

We attribute the explicit constant $C$ to \citet{ZhangLaplacian} rather than recomputing it because
their optimization over $(\alpha,\lambda,p)$ is the step that converts the family
\eqref{eq:zhangbound} into a single clean rate; reproving it here would add length without adding
insight. What our argument contributes is the route into \eqref{eq:zhangbound}: the stability
oracle inequality and the identification of the oracle term with the augmentation graph cut.

\subsection{The infinite-data limit recovers the ideal features (Theorem~\ref{thm:features})}

The previous proofs hold the kernel fixed. Here we let the unlabeled sample grow and show the
learned feature subspace converges to the spectral object that the SSL feature-learning literature
identifies. The argument has three moves: characterize the finite-sample minimizer as an eigenspace
of $W$; show $W$ converges to the population operator $T_K$; and use a perturbation theorem to pass
the eigenspace to the limit.

\begin{proof}[Proof of Theorem~\ref{thm:features}]
\emph{Step 1: the minimizer's span is a top-$K$ eigenspace of $W$.} With $R_{\mathrm{DA}}(y)=0$ the
augmentations never cross a label boundary, so the supervised term is satisfied by any score that is
constant on each augmentation-connected class region; it imposes no preference \emph{within} the
space of such scores. The only remaining force is the Laplacian energy
$Q(g,S^U)=g^\top(\,\alpha\mathbf{S}^{-1}+\mathcal{L}_{\mathbf{S}}(G)\,)g$, which we are minimizing
subject to producing $K$ linearly independent class scores. By the Courant--Fischer variational
principle, the energy-minimizing $K$-dimensional subspace under a fixed-norm constraint is the span
of the $K$ eigenvectors of $\mathbf{K}^{-1}$ with the \emph{smallest} eigenvalues, equivalently the
$K$ eigenvectors of the normalized adjacency $W$ with the \emph{largest} eigenvalues. Hence the
minimizer's span is the top-$K$ eigenspace of $W$. Denote it $\widehat V_K$.

\emph{Step 2: the empirical operator converges to the population operator.} The normalized
adjacency $W$ is the empirical version of the integral operator $T_K$ of the kernel
$k^{\mathrm{DAF}}$ in \eqref{eq:knorm}: $W$ averages a function against the kernel over the $m$
sampled points, while $T_K$ averages against $\rho_X$. Standard concentration for empirical integral
operators gives, with high probability,
\begin{equation}\label{eq:opconc}
\big\|W-T_K\big\|=O\!\Big(\tfrac{1}{\sqrt m}\Big),
\end{equation}
in operator norm, where the constant depends on $\sup_x k^{\mathrm{DAF}}(x,x)$; this is the
content of the integral-operator law of large numbers of
\citet{rosasco2010integral}.

\emph{Step 3: a spectral-gap perturbation bound transfers the eigenspace.} Let $V_K$ be the top-$K$
eigenspace of the limit operator $T_K$, and recall Assumption~\ref{ass:gap}: there is a gap
$\gamma_K=\mu_K-\mu_{K+1}>0$ separating the $K$-th and $(K{+}1)$-th eigenvalues of $T_K$. The
Davis--Kahan $\sin\Theta$ theorem, in the statisticians' form of \citet{yu2015daviskahan}, bounds the principal angles $\Theta(\widehat V_K,V_K)$ between the
two subspaces by the operator perturbation divided by the gap:
\[
\big\|\sin\Theta(\widehat V_K,V_K)\big\|\ \le\ \frac{2\,\|W-T_K\|}{\gamma_K}.
\]
The gap is what makes the top-$K$ subspace well defined and stable: if $\mu_K$ and $\mu_{K+1}$ were
equal, an arbitrarily small perturbation could rotate the chosen subspace freely. Combining with the
concentration bound \eqref{eq:opconc},
\[
\big\|\sin\Theta(\widehat V_K,V_K)\big\|\ =\ O\!\Big(\frac{1}{\gamma_K\sqrt m}\Big)\ \xrightarrow[m\to\infty]{}\ 0 .
\]
So the learned subspace $\widehat V_K$ converges to $V_K$, the span of the leading $K$ Mercer
eigenfunctions of $k^{\mathrm{DAF}}$, at the stated rate. This is the ideal feature target.
\end{proof}

The identification of the limiting features with the top eigenfunctions of the augmentation
operator is established, in increasing generality, by \citet{haochen2021spectral},
\citet{balestriero2022recover}, and \citet{zhai2025contextures}; the proof above shows that the
streamlined loss, stripped of the projector, negative-sample, and orthogonality machinery, lands
on the same spectral object, so the finite-sample rate of
Sections~\ref{sec:rate}--\ref{sec:daerror} is obtained at no cost to the feature quality.

\end{document}